\documentclass{article}
\usepackage{arxiv}
\usepackage[toc,page]{appendix}

\usepackage[utf8]{inputenc} %
\usepackage[T1]{fontenc}    %
\usepackage{hyperref}       %
\usepackage{url}            %
\usepackage{booktabs}       %
\usepackage{amsfonts}       %
\usepackage{nicefrac}       %
\usepackage{microtype}      %
\usepackage{lipsum}
\usepackage{mathtools}
\usepackage{cite}
\usepackage{amsmath}
\usepackage{amssymb}
\usepackage{amsthm}
\usepackage{algorithm,algcompatible}
\usepackage[toc,page]{appendix}
\newcommand{\iu}{\mathrm{i}\mkern1mu}

\newcommand{\conv}{\text{conv}}
\usepackage{lineno}

\newtheorem{theorem}{Theorem}
\newtheorem{lemma}{Lemma}

\newtheorem{corollary}{Corollary}
\newtheorem{proposition}{Proposition}

\usepackage{mathptmx}      %
\begin{document}

\title{Characterization of the Variation Spaces Corresponding to Shallow Neural Networks
}

\author{Jonathan W. Siegel \\
  Department of Mathematics\\
  Pennsylvania State University\\
  University Park, PA 16802 \\
  \texttt{jus1949@psu.edu} \\
  \And Jinchao Xu \\
  Department of Mathematics\\
  Pennsylvania State University\\
  University Park, PA 16802 \\
  \texttt{jxx1@psu.edu} \\
}

\maketitle

\begin{abstract}
We study the variation space corresponding to a dictionary of functions in $L^2(\Omega)$ for a bounded domain $\Omega\subset \mathbb{R}^d$. Specifically, we compare the variation space, which is defined in terms of a convex hull with related notions based on integral representations. This allows us to show that three important notions relating to the approximation theory of shallow neural networks, the Barron space, the spectral Barron space, and the Radon BV space, are actually variation spaces with respect to certain natural dictionaries.


\end{abstract}

\section{Introduction}

In this work we consider the variation space with respect to a dictionary $\mathbb{D}\subset H$ in a separable Hilbert space $H$. This notion arises in the study of non-linear approximation by an expansion of dictionary elements \cite{devore1998nonlinear,barron2008approximation}. Suppose that $\sup_{d\in \mathbb{D}} \|d\|_H = K_\mathbb{D} < \infty$, and consider the variation norm \cite{kurkova2001bounds,kurkova2002comparison} of $\mathbb{D}$ defined by
\begin{equation}\label{norm-definition}
 \|f\|_{\mathbb{D}} = \inf\left\{c > 0:~f/c\in \overline{\conv(\pm\mathbb{D})}\right\}
\end{equation}
This is the gauge, or Minkowski functional of the closed symmetric convex hull of $\mathbb{D}$
\begin{equation}\label{unit-ball-definition}
 \overline{\conv(\pm\mathbb{D})} = \overline{\left\{\sum_{j=1}^n a_jh_j:~n\in \mathbb{N},~h_j\in \mathbb{D},~\sum_{i=1}^n|a_i|\leq 1\right\}}.
\end{equation}
The variation space $\mathcal{K}(\mathbb{D})$ is then given by 
\begin{equation}\label{space-definition}
\mathcal{K}(\mathbb{D}) := \{f\in H:~\|f\|_{\mathbb{D}} < \infty\}.
\end{equation}

The varation norm and variation space have been introduced in different forms in the literature and play an important role in the approximation theory of neural networks \cite{barron1993universal,makovoz1996random,makovoz1998uniform,klusowski2018approximation,siegel2021sharp}, the convergence theory of greedy algorithms \cite{barron2008approximation,devore1996some,temlyakov2008greedy,temlyakov2011greedy,sil2004rate,livshits2009lower}
and in non-linear approximation
 \cite{devore1998nonlinear,kurkova2001bounds,kurkova2002comparison}. %

In this work, we begin by developing the basic properties of the variation space and variation norm. Specifically, we show that the set $\mathcal{K}(\mathbb{D})$ is a Banach space with the $\|\cdot\|_\mathbb{D}$-norm. Next, we study the variation space $\mathcal{K}(\mathbb{D})$ for the following two dictionaries which arise in the study of shallow neural networks and compare them with related notions in the approximation theory of shallow neural networks.

The first type of dictionary arises when studying neural networks with ReLU$^k$ activation function \cite{siegel2021sharp} $$\sigma_k(x) = \text{ReLU}^k(x) := [\max(0,x)]^k.$$ Here when $k=0$, we interpret $\sigma_k(x)$ to be the Heaviside function. Let $\Omega \subset \mathbb{R}^d$ be a compact domain and consider the dictionary
\begin{equation}\label{relu-k-space-definition}
 \mathbb{P}_k = \{\sigma_k(\omega\cdot x + b):~\omega\in S^{d-1},~b\in [c_1,c_2]\}\subset L^2(\Omega),
\end{equation}
where $S^{d-1} = \{\omega\in \mathbb{R}^d:~|\omega| = 1\}$ is the unit sphere and $c_1$ and $c_2$ are chosen to satisfy 
\begin{equation}
 c_1 < \inf \{x\cdot \omega:x\in \Omega, \omega\in S^{d-1}\} < \sup\{x\cdot \omega:x\in \Omega, \omega\in S^{d-1}\}< c_2.
\end{equation}
The important point is that $\sigma_k(\omega\cdot x + b)$ for all planes $\omega\cdot x + b = 0$ which intersect $\Omega$ must be strictly contained in $\mathbb{P}_k$. Note that we are suppressing the dependence on the domain $\Omega$ and dimension $d$ for notational convenience. We explain where this definition comes from in Section \ref{relu-barron-section}. 

Recently, neural networks with ReLU activation function have shown remarkable empirical success on problems in computer vision and natural language processing \cite{lecun2015deep}. A shallow neural network of width $n$ with ReLU activation function is a function of the form
\begin{equation}
 f_n(x) = \sum_{i=1}^n a_i\sigma_1(\omega_i \cdot x + b_i),
\end{equation}
for some parameters $a_i,b_i\in \mathbb{R}$ and $\omega_i\in \mathbb{R}^d$, where $\sigma_1(x) = \max(0,x)$ is the rectified linear unit \cite{nair2010rectified}. A natural measure of complexity on the parameters $a_i,b_i,\omega_i$ is the squared $\ell^2$-norm
\begin{equation}
 C(f_n) := C(\{a_i,b_i,\omega_i\}_{i=1}^n) := \sum_{i=1}^n a_i^2 + \|\omega_i\|_2^2,
\end{equation}
which corresponds to the regularizer induced by the common practice of weight decay \cite{krogh1991simple}. In \cite{ongie2019function}, a semi-norm is defined by taking the complexity required to uniformly approximate $f$ on compact subsets as the width $n\rightarrow \infty$. Specifically, they define the semi-norm
\begin{equation}
 \bar{R}(f) = \lim_{\epsilon\rightarrow 0}\inf\left\{C(f_n)~\text{s.t.}~|f_n(x) - f(x)|\leq \epsilon,~\text{for}~|x|\leq \epsilon^{-1}\right\}.
\end{equation}
Functions for which $\bar{R}(f)$ is finite can be approximated arbitrarily closely by shallow ReLU neural networks with bounded complexity. It is shown in \cite{ongie2019function} that this semi-norm is given by
\begin{equation}
 \bar{R}(f) = \min\left\{\|\alpha\|_1,~\text{s.t.}~f(x) = \int_{S^{d-1}\times \mathbb{R}} [\sigma_1(\omega\cdot x + b) - \sigma_1(b)]d\alpha(\omega,b) + c\right\},
\end{equation}
where the infemum is taken over all signed Borel measures $\alpha$ and constants $c$, and $\|\alpha\|_1$ denotes the total variation norm of $\alpha$.  
Further, they provide a characterization of the semi-norm $\bar{R}$ in terms of the Radon transform, showing that (roughly speaking, see \cite{ongie2019function}, Theorem 2)
\begin{equation}\label{radon-equation-1}
 \bar{R}(f) = \gamma_d\|\mathcal{R}(\Delta^{\frac{d+1}{2}}f)\|_1 + |\nabla f(\infty)|,
\end{equation}
where $\mathcal{R}$ is the Radon transform, $\gamma_d$ is a dimension dependent constant, and the gradient at $\infty$ is defined by $\nabla f(\infty) = \lim_{r\rightarrow\infty} \frac{1}{r^{d-1}|S^{d-1}|}\int_{|x| = r} \nabla f(x)dx$. Note that one part of this equality, namely that the left hand side is less than or equal to the right, was also proved in \cite{petrosyan2020neural}. When $d$ is even, the fractional power of the Laplacian appearing in \eqref{radon-equation-1} must be defined in terms of the ramp filter in the Radon domain (see \cite{ongie2019function,parhi2020banach} for details).

This notion is extended the higher powers of the ReLU, i.e. to $\sigma_k = [\max(0,x)]^k$ for $k\geq 2$ in \cite{parhi2021kinds,parhi2020banach} (the case $k=0$ was treated in \cite{kainen2010integral}). They propose a family of seminorms, called the Radon BV semi-norms, denoted by $|f|_{(k)}$, indexed by $k$ and defined via the Radon transform $\mathcal{R}$ (again, roughly speaking, see \cite{parhi2020banach} for details) as
\begin{equation}
 |f|_{(m)} = \gamma_d\|\mathcal{R}(\Delta^{\frac{d+2k-1}{2}}f)\|_1.
\end{equation}
Further, they prove a representer theorem for this semi-norm, i.e. they show that minimizers to the regularized problem
\begin{equation}
 \arg\min_f \sum_{i=1}^N\ell(f(x_i),y_i) + \gamma|f|_{(m)},
\end{equation}
where $(x_1,y_1),...,(x_N,y_N)$ is a finite data sample are shallow neural networks with ReLU$^k$ activation function and finite width (depending on $N$).

A closely related notion introduced recently concerning approximation by shallow ReLU neural networks is the Barron norm \cite{ma2019barron,weinan2021barron}, defined by
\begin{equation}
 \|f\|_{\mathcal{B}} = \min\left\{\mathbb{E}_\rho(|a|(|\omega|_1 + |b|)):~f(x) = \int_{\mathbb{R}\times\mathbb{R}^d\times\mathbb{R}} a\sigma_1(\omega\cdot x + b)\rho(da,d\omega,db)\right\},
\end{equation}
where the infemum is over all probability measures $\rho$ on $\mathbb{R}\times\mathbb{R}^d\times\mathbb{R}$. The properties of this norm have also been studied in \cite{wojtowytsch2020representation}, for instance.

Our contribution is to show that on a bounded domain $\Omega$, the notion of Barron norm coincides with the variation norm of the dictionary $\mathbb{P}_1$. Specifically, we show the equivalence
\begin{equation}
 \|f\|_{\mathbb{P}_1} \eqsim \|f\|_{\mathcal{B},\Omega} := \inf_{f_e|\Omega = f}\|f_e\|_{\mathcal{B}}.
\end{equation}
Here the constant in the above bound sclaes with the square root of the dimension and is due to the fact that the Barron norm measures the norm of $\omega$ in $\ell^1$ while we measure it in $\ell^2$.

In addition, we show that up to a polynomial kernel, the Radon BV semi-norm coincides with the variation norm of $\mathbb{P}_k$ on a bounded domain $\Omega$.
Specifically, we have for $k\geq 0$
\begin{equation}\label{quotient-equivalent-1}
 \inf_{p\in \mathcal{P}_k}\|f + p\|_{\mathbb{P}_k} \eqsim |f|_{(k+1),\Omega} := \inf_{f_e|\Omega = f}|f_e|_{(k+1)},
\end{equation}
where $\mathcal{P}_k$ is the space of polynomials of degree at most $k$ and the infimum is taken over all extensions $f_e$ of $f$ to the whole of $\mathbb{R}^d$.

By equivalence of the $\bar{R}$ semi-norm and Radon BV semi-norm noted in \cite{parhi2020banach} this also implies that
\begin{equation}\label{quotient-equivalent-2}
 \inf_{p\in \mathcal{P}_k}\|f + p\|_{\mathbb{P}_1} \eqsim \inf_{f_e|\Omega = f}\bar{R}(f_e).
\end{equation}
The constants implicit in the equivalences \eqref{quotient-equivalent-1} and \eqref{quotient-equivalent-2} do not depend upon the dimension. We also prove the equivalences
\begin{equation}
    \|f\|_{\mathbb{P}_k} \eqsim |f|_{(k+1),\Omega} + \|f\|_{L^2(\Omega)},~\|f\|_{\mathbb{P}_1} \eqsim \inf_{f_e|\Omega = f}\bar{R}(f) + \|f\|_{L^2(\Omega)}.
\end{equation}
However, for these the implied constant does depend upon the dimension.

Finally, we also give a complete characterization of the variation space corresponding to $\mathbb{P}_k$ in one dimension. In particular, we prove that
\begin{equation}
  \|f\|_{\mathbb{P}_k} \eqsim \sum_{j=0}^{k-1} |f^{(j)}(-1)| + \|f^{(k)}\|_{BV([-1,1])}. 
 \end{equation}

The second type of dictionary related to shallow neural networks  which we consider is the spectral dictionary of order $s \geq 0$, given by
\begin{equation}
 \mathbb{F}_s = \{(1+|\omega|)^{-s}e^{2\pi \iu \omega\cdot x}:~\omega\in \mathbb{R}^d\}\subset L^2(\Omega).
\end{equation}
Our contribution is to show that the variation space of $\mathbb{F}_s$ can be completely characterized in terms of the Fourier transform. In particular, in Section \ref{spectral-barron-section} we prove that
\begin{equation}\label{barron-integral-condition}
 \|f\|_{\mathbb{F}_s} = \inf_{f_e|_{\Omega}= f} \int_{\mathbb{R}^d} (1+|\xi|)^s|\hat{f}_e(\xi)|d\xi,
\end{equation}
where the infimum is taken over all extensions $f_e\in L^1(\mathbb{R}^d)$. (Note that we have here equality, not just equivalence.) 

The norm in \eqref{barron-integral-condition}  was first introduced by Barron \cite{barron1993universal}, who showed that functions in the space $\mathcal{K}(\mathbb{F}_1)$ could be approximated with rate $O(n^{-\frac{1}{2}})$ using shallow networks with sigmoidal activation function.
These results have been extended to networks with ReLU$^k$ activation functions in \cite{klusowski2018approximation,CiCP-28-1707}. The spectral Barron norm \eqref{barron-integral-condition} has also been important in understanding the approximation properties of shallow neural networks with more general activation functions \cite{siegel2020approximation,hornik1994degree}. Our contribution is to show that the spectral Barron norm \eqref{barron-integral-condition} is equivalent to the variation norm with respect to a suitable dictionary of decaying Fourier modes.

The paper is organized as follows. In Section \ref{basic-properties-section} we discuss the basic properties of the variation spaces. In particular, show that they are Banach spaces. In Section \ref{relu-barron-section}, we analyze the spaces $\mathcal{K}(\mathbb{P}_k)$. We show that when $k=1$, the space is equivalent to the Barron space studied in \cite{ma2019barron} and also compare them with the Radon BV spaces. In Section \ref{one-dimensional-characterization}, we give a characterization of $\mathcal{K}(\mathbb{P}_k)$ when $d=1$ in terms of the space of bounded variation. Then, in Section \ref{spectral-barron-section}, we give a characterization of $\mathcal{K}_1(\mathbb{F}^d_s)$ in terms of the Fourier transform, showing that it is equivalent to the spectral Barron norm. Finally, we give some concluding remarks and further research directions.

\section{Basic Properties of $\mathcal{K}_1(\mathbb{D})$}\label{basic-properties-section}
Let us first develop the elementary properties of the variation space $\mathcal{K}(\mathbb{D})$. The key result is that $\mathcal{K}(\mathbb{D})$ is a Banach space with the variation norm $\|\cdot\|_{\mathbb{D}}$.
\begin{lemma}\label{fundamental-norm-lemma}
 Suppose that $\sup_{d\in \mathbb{D}} \|d\|_H = K_\mathbb{D} < \infty$. Then the $\mathcal{K}_1(\mathbb{D})$ norm satisfies the following properties.
 \begin{itemize}
  \item $\overline{\conv(\pm\mathbb{D})} = \{f\in H:\|f\|_{\mathcal{K}_1(\mathbb{D})}\leq 1\}$
  \item $\|f\|_H\leq K_\mathbb{D}\|f\|_{\mathcal{K}_1(\mathbb{D})}$
  \item $\mathcal{K}_1(\mathbb{D}) := \{f\in H:~\|f\|_{\mathcal{K}_1(\mathbb{D})} < \infty\}$ is a Banach space with the $\|\cdot\|_{\mathcal{K}_1(\mathbb{D})}$ norm
 \end{itemize}
 
\end{lemma}
\begin{proof}
 The first two statements are well-known and can be found for instance in \cite{kurkova2001bounds}.
 For the third statement we must show that the set $\mathcal{K}_1(\mathbb{D})$ is complete with respect to the $\|\cdot\|_{\mathcal{K}_1(\mathbb{D})}$ norm.
 
 Let $\{f_n\}_{n=1}^\infty$ be a Cauchy sequence with respect to the $\|\cdot\|_{\mathcal{K}_1(\mathbb{D})}$ norm. From the second statement, we have $\|f_n - f_m\|_H\leq K_\mathbb{D}\|f_n - f_m\|_{\mathcal{K}_1(\mathbb{D})}$, so that the sequence is Cauchy with respect the the $H$-norm as well. Thus, there exists an $f\in H$, such that $f_n\rightarrow f$ in $H$, i.e. such that $\|f_n - f\|_H\rightarrow 0$.
 
 We will show that also $\|f_n - f\|_\mathbb{D}\rightarrow 0$, i.e. that we have convergence in the variation norm as well (note that this automatically implies that $\|f\|_{\mathbb{D}}<\infty$). 
 
 To this end, let $\epsilon > 0$ and choose $N$ such that $\|f_n - f_m\|_{\mathbb{D}} < \epsilon / 2$ for $n,m \geq N$ ($\{f_n\}$ is Cauchy, so this is possible). In particular, this means that $\|f_N - f_m\|_{\mathbb{D}}\leq \epsilon / 2$ for all $m > N$. Now the first statement implies that $f_m - f_N \in (\epsilon / 2)\overline{\conv(\pm\mathbb{D})}$, or in other words that $f_m \in f_N + (\epsilon / 2)\overline{\conv(\pm\mathbb{D})}$. Since $f_m\rightarrow f$ in $H$, and $\overline{\conv(\pm\mathbb{D})}$ is closed in $H$ by definition, we get $f\in f_N + (\epsilon / 2)\overline{\conv(\pm\mathbb{D})}$. Hence $\|f - f_N\|_{\mathbb{D}} \leq \epsilon / 2$ and the triangle inequality finally implies that $\|f - f_m\|_{\mathbb{D}} \leq \epsilon$ for all $m \geq N$. Thus $f_n\rightarrow f$ in the variation norm and $\mathcal{K}(\mathbb{D})$ is complete.
\end{proof}

Let us remark that for some dictionaries $\mathbb{D}$ the $\mathcal{K}(\mathbb{D})$ space can be substantially smaller than $H$. In fact, if the dictionary $\mathbb{D}$ is contained in a closed subspace of $H$, then we have the following elementary result.
\begin{lemma}\label{subspace-lemma}
 Let $K\subset H$ be a closed subspace of $H$. Then $\mathbb{D}\subset K$ iff $\mathcal{K}_1(\mathbb{D})\subset K$.
\end{lemma}
\begin{proof}
 We have $\mathbb{D}\subset\mathcal{K}(\mathbb{D})$ so that the reverse implication is trivial. For the forward implication, since $\mathbb{D}\subset K$ and $K$ is closed, it follows that $\overline{\conv(\pm\mathbb{D})}\subset K$. Then, from the definition \eqref{norm-definition}, it follows that
 \begin{equation}
  \mathcal{K}(\mathbb{D}) = \bigcup_{r > 0} r\cdot \overline{\conv(\pm\mathbb{D})}\subset K.
 \end{equation}

\end{proof}
 A simple example when this occurs is when considering a shallow neural network with activation function $\sigma$ which is a polynomial of degree $k$. In this case the space $\mathcal{K}(\mathbb{D})$ is contained in the finite-dimensional space of polynomials of degree $k$, and the $\|\cdot\|_{\mathbb{D}}$ norm is infinite on non-polynomial functions. This is related to the well-known result that neural network functions are dense iff the activation function is not a polynomial \cite{leshno1993multilayer}. 
\begin{proposition}
 Let $\Omega\subset \mathbb{R}^d$ be a bounded domain and $\mathbb{D} = \{\sigma(\omega\cdot x + b):(\omega,b)\in \mathbb{R}^d\times \mathbb{R}\}\subset L^2(\Omega)$, where the activation function $\sigma\in L^\infty_{loc}(\mathbb{R})$, i.e. $\|\sigma\|_{L^\infty(K)} < \infty$ for any compact set $K\subset \mathbb{R}$. Suppose further that the set of discontinuities of $\sigma$ has Lebesgue measure $0$. Then $\mathcal{K}(\mathbb{D})$ is finite dimensional iff $\sigma$ is a polynomial (a.e.).
\end{proposition}
\begin{proof}
 If $\sigma$ is a polynomial, $\mathbb{D}$ is contained in the space of polynomials of degree at most $\text{deg}(\sigma)$, which is finite dimensional. This implies the result by Lemma \ref{subspace-lemma}. For the reverse implication, we use Theorem 1 of \cite{leshno1993multilayer}, which states that if $\sigma$ is not a polynomial, then 
 $$
 C(\Omega) \subset \overline{\left\{\sum_{i=1}^na_i\sigma(\omega_i\cdot x + b_i)\right\}},
 $$
 where the closure is taken in $L^\infty(\Omega)$ (note that this cumbersome statement is necessary since $\sigma$ may not be continuous). This immediately implies that $\mathcal{K}_1(\mathbb{D})$ is dense in $L^2(\Omega)$ (since $C(\Omega)$ is dense in $L^2(\Omega)$), and thus obviously not finite dimensional.
\end{proof}
While in this example the variation norm is finite dimensional, this is typically not the case for dictionaries of interest. Specifically for the dictionaries $\mathbb{P}_k$ and $\mathbb{F}_s$ which we study, this space is infinite dimensional but still much smaller than $H$. The size of the variation space has been precisely quantified for $\mathbb{P}_k$ in terms of the metric entropy in \cite{siegel2021sharp} and these spaces have been used as trial spaces for solving PDEs in \cite{hao2021efficient}. However, it remains an interesting open question what the practical utility of $\mathcal{K}(\mathbb{P}_k)$ and $\mathcal{K}(\mathbb{F}_s)$ really are.

The significance of the variation norm is that functions $f\in \mathcal{K}(\mathbb{D})$ can be efficiently approximated by convex combinations of small numbers of dictionary elements. In particular, we have the following result of Maurey \cite{pisier1981remarques,jones1992simple,barron1993universal}. Denote
\begin{equation}
 \Sigma_{n,M}(\mathbb{D}) = \left\{\sum_{j=1}^n a_jh_j:~h_j\in \mathbb{D},~\sum_{i=1}^n|a_i|\leq M\right\}.
\end{equation}
Then we have the following result.
\begin{theorem}[Lemma 1 in \cite{barron1993universal}]
 Suppose that $f\in \mathcal{K}(\mathbb{D})$. Then for $M = \|f\|_{\mathbb{D}}$ we have
\begin{equation}
 \inf_{f_n\in \Sigma_{n,M}(\mathbb{D})} \|f - f_n\|_H \leq K_\mathbb{D}\|f\|_{\mathcal{K}(\mathbb{D})}n^{-\frac{1}{2}},
\end{equation}

\end{theorem}

We note the following simply converse to Maurey's approximation rate. In particular, if a function can be approximated by elements from $\Sigma_{n,M}(\mathbb{D})$ with fixed $M$, then it must be in the space $\mathcal{K}(\mathbb{D})$.
\begin{proposition}
 Let $H$ be a Hilbert space and $f\in H$. Suppose that $f_n\rightarrow f$ in $H$ with $f_n\in \Sigma_{n,M}(\mathbb{D})$ for a fixed $M < \infty$. Then $f\in \mathcal{K}(\mathbb{D})$ and
 \begin{equation}
  \|f\|_{\mathbb{D}} \leq M.
 \end{equation}

\end{proposition}
\begin{proof}
 It is clear that we must only prove this for $M=1$. From the definitions we have $\Sigma_{n,1}(\mathbb{D}) \subset \overline{\conv(\pm\mathbb{D})}$ for every $n$. Thus $f_n\in \overline{\conv(\pm\mathbb{D})}$ and since $\overline{\conv(\pm\mathbb{D})}$ is closed, we get $f\in \overline{\conv(\pm\mathbb{D})}$, so that $\|f\|_{\mathbb{D}} \leq 1$, as desired.
\end{proof}

Next, we wish to connect the space $\mathcal{K}(\mathbb{D})$ defined via the closed symmetric convex hull of $\mathbb{D}$ to integral representations, which have recently become a popular concept in the approximation theory of shallow neural networks \cite{weinan2019barron,wojtowytsch2020representation,ongie2019function,parhi2020banach}. An integral representation of a function $f$ over the dictionary $\mathbb{D}$ is given by
\begin{equation}\label{integral-representation-def}
    f = \int_\mathbb{D} i_{\mathbb{D}\rightarrow H}d\mu.
\end{equation}
Here the dictionary $\mathbb{D}$ inherits the subspace topology from the Hilbert space $H$, $d\mu$ is a (signed) Borel measure with finite variation on $\mathbb{D}$, i.e.
\begin{equation}
    \|\mu\| = \sup_{\substack{g:~\mathbb{D}\rightarrow [-1,1]\\g~\text{measurable}}} \int_\mathbb{D} gd\mu < \infty,
\end{equation}
and the integral is the Bochner integral of the inclusion map $i_{\mathbb{D}\rightarrow H}:\mathbb{D}\rightarrow H$. Note that since $H$ is separable, the inclusion map is $\mu$-measurable by the Pettis measurability theorem. Further, if $\mathbb{D}$ is bounded, i.e. if $|\mathbb{D}| = \sup_{d\in \mathbb{D}} \|d\|_H < \infty$, then since $\mu$ has finite variation the inclusion map $i_{\mathbb{D}\rightarrow H}$ is absolutely integrable and so the Bochner integral exists (see \cite{diestel2012sequences}, Chapter 4).

We prove that if the dictionary $\mathbb{D}$ is compact, then membership in $\mathcal{K}(\mathbb{D})$ is equivalent to the existence of an integral representation.
\begin{lemma}\label{prokhorov-lemma}
 Suppose that $\mathbb{D}\subset H$ is compact. Then $f\in \mathcal{K}(\mathbb{D})$ iff there exists a Borel measure $\mu$ on $\mathbb{D}$
 \begin{equation}
  f = \int_\mathbb{D} i_{\mathbb{D}\rightarrow H}d\mu.
 \end{equation}
 Moreover,
 \begin{equation}
  \|f\|_{\mathbb{D}} = \inf\left\{\|\mu\|:~f = \int_\mathbb{D} i_{\mathbb{D}\rightarrow H}d\mu\right\}.
 \end{equation}

\end{lemma}
\begin{proof}
 From the definition of the variation norm \eqref{norm-definition} we must show that 
 $$\overline{\conv(\pm\mathbb{D})} = M(\mathbb{D}):=\left\{\int_\mathbb{D} i_{\mathbb{D}\rightarrow H}d\mu:~\|\mu\| \leq 1\right\}.$$
 We first show that $M(\mathbb{D})\subset \overline{\conv(\pm\mathbb{D})}$. The idea of the proof is to approximate the inclusion map $i_{\mathbb{D}\rightarrow H}$ by a simple function. The only technical issue is that we must be able to restrict the range of this simple function to lie in $\mathbb{D}$. We proceed as in the proof of Bochner's theorem (see \cite{diestel2012sequences}, Chapter IV) with minor modification.
 
 Let $\mu$ be a Borel measure on $\mathbb{D}$ with  variation $\|\mu\|\leq 1$. Since $H$ is separable, the Pettis measurability theorem implies that the inclusion $i_{\mathbb{D}\rightarrow H}$ is $\mu$-measurable. So for each $n$ we can choose a countably valued $\mu$-measurable function $f_n$ such that $\|f_n - i_{\mathbb{D}\rightarrow H}\|_H\leq 1/2n$ $\mu$-almost everywhere. Thus we can write
 \begin{equation}
     f_n = \sum_{k=1}^\infty a_{n,k}\chi_{E_{n,k}}
 \end{equation}
 for elements $a_{n,k}\in H$ and $\mu$-measurable sets $E_{n,k}$ which satisfy $E_{n,i}\cap E_{n,j} = \emptyset$ when $i\neq j$. The condition $\|f_n - i_{\mathbb{D}\rightarrow H}\|_H\leq 1/2n$ means that for every $d\in E_{n,k}\subset \mathbb{D}$ we have $\|a_{n,k} - d\|_H\leq 1/2n$. Using the triangle inequality this means that for any $d,d'\in E_{n,k}$, we have $\|d - d'\|_H \leq 1/n$. Now for each $E_{n,k}$ we choose $d_{n,k}\in E_{n,k}$ and set
 \begin{equation}
     \tilde{f}_n = \sum_{k=1}^\infty d_{n,k}\chi_{E_{n,k}}.
 \end{equation}
 Then we have $\|\tilde{f}_n-i_{\mathbb{D}\rightarrow H}\|_H \leq 1/n$ $\mu$-almost everywhere and the range of $\tilde{f}_n$ lies in $H$. Finally, for each $n$ we choose $p_n$ such that
 \begin{equation}
     \int_{\left(\cup_{k=p_n+1}^\infty E_{n,k}\right)} \|\tilde{f}_n\|_H d\mu \leq \frac{1}{n}.
 \end{equation}
 Since $\mathbb{D}$ is compact and thus bounded and $\mu$ satisfies $\|\mu\|\leq 1$, the function $\|\tilde{f}_n\|_H$ is in $L^1(d\mu)$ so that such a $p_n$ can always be chosen. We now set
 \begin{equation}
     g_n = \sum_{k=1}^{p_n} d_{n,k}\chi_{E_{n,k}}.
 \end{equation}
 Then $g_n$ is a simple function satisfying $\int_{\mathbb{D}} \|i_{\mathbb{D}\rightarrow H} - g_n\|_Hd\mu \leq \frac{2}{n}$. This means that
 \begin{equation}
     \left|\int_\mathbb{D} i_{\mathbb{D}\rightarrow H}d\mu - \sum_{k=1}^{p_n}d_{n,k}\mu(E_{n,k})\right| \leq \frac{2}{n}.
 \end{equation}
 By design, $d_{n,k}\in \mathbb{D}$ and since $\|\mu\|\leq 1$, we get $\sum_{k=1}^{p_n}|\mu(E_{n,k})| \leq 1$. Thus
 \begin{equation}
     \sum_{k=1}^{p_n}d_{n,k}\mu(E_{n,k})\in \overline{\conv(\pm\mathbb{D})}
 \end{equation}
 for every $n$. Letting $n\rightarrow \infty$, we see that
 \begin{equation}
     \int_\mathbb{D} i_{\mathbb{D}\rightarrow H}d\mu\in \overline{\conv(\pm\mathbb{D})}.
 \end{equation}
 Since $\mu$ was an arbitrary measure we get $M(\mathbb{D})\subset \overline{\conv(\pm\mathbb{D})}$.
 
 Next we prove the reverse inclusion. Given any convex combination 
 $$f = \sum_{i=1}^Na_id_i,$$
 with $d_i\in \mathbb{D}$ and $\sum_{i=1}^N|a_i| \leq 1$, we can choose $\mu = \sum_{i=1}^\infty a_i\delta_{d_i}$ to be a linear combination of Dirac deltas to see that $f\in M(\mathbb{D})$. To complete the proof we must show that $M(\mathbb{D})$ is closed. We will prove this using Prokhorov's theorem \cite{prokhorov1956convergence} (see also \cite{dudley2018real}, Theorem 11.5.4, for instance). Let $f_n\rightarrow f$ with $f_n\in M(\mathbb{D})$ and let $\mu_n$ be the corresponding sequence of Borel measures on $\mathbb{D}$ such that
 \begin{equation}
     f_n = \int_\mathbb{D} i_{\mathbb{D}\rightarrow H}d\mu_n
 \end{equation}
 and $\|\mu_n\|\leq 1$. By the compactness of $\mathbb{D}$ and Prokhorov's theorem, by taking a subsequence if necessary we may assume that the $\mu_n\rightarrow \mu$ weakly, i.e. that the integrals against continuous functions on $\mathbb{D}$ converges. Set $\tilde{f} = \int_\mathbb{D} i_{\mathbb{D}\rightarrow H}d\mu$, which is Bochner integrable by the comments prior to the lemma. Choose a countable dense sequence $\{\lambda_i\}_{i=1}^\infty\in H$. The weak convergence implies that
 \begin{equation}
     \lim_{n\rightarrow \infty}\left\langle \lambda_i,f_n\right\rangle_H = \left\langle \lambda_i,\tilde{f}\right\rangle_H.
 \end{equation}
 for every $i$. The strong convergence $f_n\rightarrow f$ implies the same with $f$ replacing $\tilde{f}$. Thus $\langle \lambda_i, f\rangle_H = \langle \lambda_i, \tilde{f}\rangle_H$ for all $i$. Hence $f = \tilde{f}\in M(\mathbb{D})$, as desired.
\end{proof}
Finally, we note that the compactness in the preceding theorem was necessary. Indeed, we have the following simple example.
\begin{proposition}
 Suppose that $\Omega\subset \mathbb{R}^d$ is bounded and $\sigma$ is a smooth sigmoidal function. Let $H = L^2(\Omega)$ and $\mathbb{D}_\sigma$ defined by
 \begin{equation}
     \mathbb{D}_\sigma = \{\sigma(\omega\cdot x + b),~\omega\in \mathbb{R}^d,~b\in \mathbb{R}\}.
 \end{equation}
 Then $\overline{\conv(\pm\mathbb{D}_\sigma)} \supsetneq M(\mathbb{D}_\sigma)$, where $M(\mathbb{D}_\sigma)$ is defined as in the proof of the previous lemma.
\end{proposition}
\begin{proof}
 Let $\sigma_0$ be the Heaviside activation function. Then we have
 \begin{equation}
  \lim_{r\rightarrow \infty}\|\sigma_0(x_1) - \sigma(rx_1)\|_H = 0,
 \end{equation}
 since $\sigma$ is sogmoidal.
 Thus $\sigma_0(x_1)\in \overline{\conv(\pm\mathbb{D}_\sigma)}$. However, since $\sigma$ is smooth, the discontinuous function $\sigma_0(x_1)$ cannot have an integral representation of the form \eqref{integral-representation-def}, so that $\sigma_0(x_1)\notin M(\mathbb{D}_\sigma)$.
\end{proof}

\section{Properties of $\mathcal{K}(\mathbb{P}^d_k)$ and relationship with the Barron and Badon BV spaces}\label{relu-barron-section}
In this section we study the space $\mathcal{K}(\mathbb{P}_k)$ in more detail. We begin by explaining the precise definition \eqref{relu-k-space-definition}, i.e. how we define an appropriate dictionary corresponding to the ReLU$^k$ activation function. The problem with letting $\sigma_k(x) = [\max(0,x)]^k$ and setting
\begin{equation}
 \mathbb{D} = \{\sigma_k(\omega\cdot x + b):~\omega\in \mathbb{R}^d,~b\in \mathbb{R}\},
\end{equation}
is that unless $k=0$ the dictionary elements are not bounded in $L^2(B_1^d)$, since $\sigma_k$ is not bounded and we can shift $b$ arbitrarily. This manifests itself in the fact that $\|\cdot\|_{\mathcal{K}_1(\mathbb{D})}$ is a semi-norm which contains the set of polynomials of degree at most $k-1$ in its kernel (this occurs since the arbirtrarily large elements in $\mathbb{D}$ are polynomials on the doamin $\Omega$).

We rectify this issue by considering the dictionary
\begin{equation}
 \mathbb{P}_k = \{\sigma_k(\omega\cdot x + b):~\omega\in S^{d-1},~b\in [c_1,c_2]\},
\end{equation}
where $c_1$ and $c_2$ are chosen to satisfy 
\begin{equation}
 c_1 < \inf \{x\cdot \omega:x\in \Omega, \omega\in S^{d-1}\} < \sup\{x\cdot \omega:x\in \Omega, \omega\in S^{d-1}\}< c_2.
\end{equation}
This has the effect of ensuring that the dictionary $\mathbb{P}_k$ is bounded and the constants $c_1$ and $c_2$ are chosen so that $\sigma_k(\omega\cdot x + b)\in \mathbb{P}_k$ whenever the hyperplane $\{\omega\cdot x + b = 0\}$ intersects $\Omega$. Further, when $c_1 < b < \inf \{x\cdot \omega:x\in \Omega, \omega\in S^{d-1}\}$ or $\sup\{x\cdot \omega:x\in \Omega, \omega\in S^{d-1}\} < b < c_2$, we recover all polynomials of degree at most $k$ on $\Omega$ as well.

Next, we consider the relationship between $\mathcal{K}(\mathbb{P}_1)$ and the Barron norm introduced in \cite{ma2019barron}, which is given by
\begin{equation}\label{barron-norm}
 \|f\|_{\mathcal{B}} = \inf\left\{\mathbb{E}_\rho(|a|(|\omega|_1 + |b|)):~f(x) = \int_{\mathbb{R}\times\mathbb{R}^d\times\mathbb{R}} a\sigma_1(\omega\cdot x + b)\rho(da,d\omega,db)\right\},
\end{equation}
where we recall that $\sigma_1$ is the rectified linear unit and the infimum is taken over all integral representations of $f$. Here $\rho$ is a probability distribution on $\mathbb{R}\times\mathbb{R}^d\times\mathbb{R}$, and the expectation is taken with respect to $\rho$. We show that the $\mathcal{K}(\mathbb{P}_1)$ space is equivalent to the Barron space when restricted to bounded domains $\Omega$. 

\begin{proposition}\label{barron-norm-equivalence-theorem}
 For any bounded domain $\Omega$, we have
 \begin{equation}
  \sqrt{d}\|f\|_{\mathbb{P}_1} \eqsim \inf_{f_e|_{\Omega} = f}\|f_e\|_{\mathcal{B}},
 \end{equation}
 where the infemum is taken over all extensions of $f$ to the whole of $\mathbb{R}^d$. Here the implied constant depends only upon the constants $c_1$ and $c_2$ taken in the definition of $\mathbb{P}_1$.
\end{proposition}
\begin{proof}
 Consider the dictionary
 \begin{equation}
 \mathbb{B} = \{(|\omega|_1 + |b|)^{-1}\sigma_1(\omega\cdot x + b):~\omega\in \mathbb{R}^d,~b\in \mathbb{R}\}\subset L^2(\Omega).
\end{equation}
From Lemma \ref{prokhorov-lemma}, it follows that $\|f\|_{\mathbb{B}} = \|f\|_{\mathcal{B}}$. Indeed, by making the change of variables $\mu = |a|(|\omega|_1 + |b|)\rho$, we get 
\begin{equation}
    \|f\|_{\mathcal{B}} = \inf\left\{\|\mu\|:~f = \int_\mathbb{B} i_{\mathbb{B}\rightarrow L^2(\Omega)}d\mu\right\}.
\end{equation}

 Thus, it suffices to show that $\mathbb{P}^d_1\subset C\sqrt{d}\cdot\overline{\conv(\pm\mathbb{B})}$ and $\mathbb{B}\subset C\cdot\overline{\conv(\pm\mathbb{P}_1)}$ for a constant $C(c_1,c_2)$.
 
 So let $g\in \mathbb{P}^d_1$. This means that $g(x) = \sigma_1(\omega \cdot x + b)$ for some $\omega\in S^{d-1}$ and $b\in [-c_1,c_2]$. Thus $$(|\omega|_1 + |b|) \leq (\sqrt{d} + \max(c_1,c_2)) \leq C(c_1,c_2)\sqrt{d}$$
 and since $(|\omega|_1 + |b|)^{-1}\sigma_1(\omega \cdot x + b)\in \mathbb{B}$, we see that $g\in C\sqrt{d}\cdot \overline{\conv(\pm\mathbb{B})}$.
 
 Now, let $g\in \mathbb{B}$. Then $g(x) = (|\omega|_1 + |b|)^{-1}\sigma_1(\omega \cdot x + b)$ for some $\omega\in \mathbb{R}^d$ and $b\in \mathbb{R}$. 
 
 Consider first the case when $\omega \neq 0$. Note that by the positive homogeneity of $\sigma_1$ we can assume that $|\omega| = 1$, i.e. that $\omega\in S^{d-1}$. Further, we have that $(|\omega|_1 + |b|)^{-1} \leq (1+|b|)^{-1}$. Thus, we must show that
 \begin{equation}
  \tilde g(x) := (1+|b|)^{-1}\sigma_1(\omega \cdot x + b)\in C\cdot\overline{\conv(\pm\mathbb{P}_1)} 
 \end{equation}
 for $\omega\in S^{d-1}$ and $b\in \mathbb{R}$. For $b\in [c_1,c_2]$ this clearly holds with $C=1$ since $(1 + |b|)^{-1} \leq 1$ and for such values of $b$, we have $\sigma_1(\omega\cdot x + b)\in \mathbb{P}^d_1$. If $b < c_1$, then $\tilde g(x) = 0$, so we trivially have $\tilde g\in \overline{\conv(\pm\mathbb{P}_1)}$. Finally, if $b > c_2$, then $\omega\cdot x + b$ is positive on $\Omega$, so that
 $$
 \tilde g(x) = (1+|b|)^{-1}(\omega\cdot x + b) = (1+|b|)^{-1}\omega\cdot x + b(1+|b|)^{-1}.
 $$
 Now $\omega\cdot x\in 2\cdot\overline{\conv(\pm\mathbb{P}_1)}$ and $1 = [\sigma_1(\omega\cdot x + 2) - \sigma_1(\omega\cdot x + 1)]\in 2\cdot \overline{\conv(\pm\mathbb{P}_1)}$.
 Combined with the above and the fact that $(1+|b|)^{-1},|b|(1+|b|)^{-1}\leq 1$, we get $\tilde g\in 4\cdot \overline{\conv(\pm\mathbb{P}_1)}$.
 
 Finally, if $\omega = 0$, then $g(x) = 1$ and by the above paragraph we clearly also have $g\in 2\cdot \overline{\conv(\pm\mathbb{P}_1)}$. This completes the proof.
\end{proof}

Note that it follows from this result that the Barron space $\mathcal{B}$ is a Banach space, which was first proven in \cite{wojtowytsch2020representation}.

Next, we compare the spaces $\mathcal{K}(\mathbb{P}_k)$ and their variation norms to the Radon BV semi-norms introduced in \cite{ongie2019function,parhi2020banach}. Specifically, these norms coincide with the $\mathbb{P}_k$-variation norms on a bounded domain $\Omega$ up to a kernel consisting of polynomials.
\begin{theorem}\label{radon-bv-equivalence-thm}
    For any bounded domain $\Omega$, we have
    \begin{equation}
    \inf_{p\in \mathcal{P}_k}\|f + p\|_{\mathbb{P}_k} = \frac{1}{k!}\inf_{f_e|\Omega = f}|f_e|_{(k+1)}
    \end{equation}
    where $|\cdot|_{(k+1)}$ is the Radon BV seminorm introduced in \cite{parhi2020banach}, $f_e$ is an extension of $f$ to the whole of $\mathbb{R}^d$, and $\mathcal{P}_k$ is the space of polynomial of degree at most $k$.
\end{theorem}
Note that by the remarks in \cite{parhi2020banach}, when $k=1$ this theorem also applies to the semi-norm introduced in \cite{ongie2019function}, which is equivalent to the Radon BV semi-norm.
\begin{proof}
 Theorem 22 in \cite{parhi2020banach} implies that $|f|_{(k+1)} \leq 1$ is equivalent to an integral representation of the form
 \begin{equation}\label{radon-intgral-representation}
     f(x) = \frac{1}{k!}\int_{S^{d-1}\times \mathbb{R}} [\sigma_{k}(\omega\cdot x + b) - (\omega\cdot x + b)^{k}]d\mu(\omega, b) + p(x),
 \end{equation}
 where $\mu$ is a Borel measure on $S^{d-1}\times \mathbb{R}$, $p(x)$ is a polynomial of degree at most $k$ and $\mu$ satisfies $\|\mu\| = 1$. 
 
 Further, Lemma \ref{prokhorov-lemma} means that $\|f\|_{\mathbb{P}_k} \leq 1$ is equivalent to the existence of an integral representation
 \begin{equation}\label{integral-representation-barron}
     f(x) = \int_{S^{d-1}\times [c_1,c_2]} \sigma_{k}(\omega\cdot x + b)d\mu(\omega, b)
 \end{equation}
 on $\Omega$, where $\mu$ is a Borel measure on $S^{d-1}\times [c_1,c_2]$ and $\|\mu\|\leq 1$. This follows since $\mathbb{P}_k$ is a compact subset of $L^2(\Omega)$ (it is continuously parameterized by the compact set $S^{d-1}\times [c_1,c_2]$).
 
 So if $\|f\|_{\mathbb{P}_k} \leq 1$ we use the integral representation and set \eqref{integral-representation-barron} and set
 \begin{equation}
     f_e(x) = k!\left[\frac{1}{k!}\int_{S^{d-1}\times [c_1,c_2]} [\sigma_{k}(\omega\cdot x + b) - (\omega\cdot x + b)^{k}]d\mu(\omega, b) + p(x)\right],
 \end{equation}
 where $$p(x) = \frac{1}{k!}\int_{S^{d-1}\times [c_1,c_2]} (\omega\cdot x + b)^{k}d\mu(\omega, b).$$
 Since $S^{d-1}\times [c_1,c_2]\subset S^{d-1}\times \mathbb{R}$ we see that $|f_e|_{k+1}\leq k!$. This implies that $\inf_{f_e|\Omega = f}|f_e|_{(k+1)} \leq k!\|f\|_{\mathbb{P}_k}$. Since $\mathcal{P}_k$ is the kernel of the $|\cdot|_{(k+1)}$ (see Lemma 19 in \cite{parhi2020banach}), we can take an infemum over $\mathcal{P}_k$ to get
 \begin{equation}
     \inf_{f_e|\Omega = f}|f_e|_{(k+1)} \leq k!\inf_{p\in \mathcal{P}_k}\|f+p\|_{\mathbb{P}_k}.
 \end{equation}
 
 For the converse, suppose that $f$ satisfies $\inf_{f_e|\Omega = f}|f_e|_{(k+1)} \leq 1$ and let $f_e$ be an extension of $f$ such that
 \begin{equation}
     |f_e|_{(k+1)} \leq 1 + \epsilon.
 \end{equation}
 We now apply integral representation \eqref{radon-intgral-representation} and
 note that if $b\notin [c_1,c_2]$, then $\sigma_{k}(\omega\cdot x + b) - (\omega\cdot x + b)^{k}$ is a polynomial on the domain $\Omega$. So we can write
 \begin{equation}
     f_e(x) = f(x) = f'(x) + q(x)
 \end{equation}
 for $x\in \Omega$, where $\|f\|_{\mathbb{P}_k}\leq 1/k!$ and $q(x)$ is a polynomial of degree at most $k$. This implies that 
 \begin{equation}
 \inf_{p\in \mathcal{P}_k}\|f+p\|_{\mathbb{P}_k} \leq \|f'\|_{\mathbb{P}_k} = 1.
 \end{equation}
 Hence $\inf_{p\in \mathcal{P}_k}\|f+p\|_{\mathbb{P}_k} \leq (1/k!)\inf_{f_e|\Omega = f}|f_e|_{(k+1)}$ as desired.
\end{proof}
As a corollary of this result, we have the following equivalence when we strengthen the Radon BV semi-norm to a norm.
\begin{corollary}
For any bounded domain $\Omega$, we have
    \begin{equation}
    \|f\|_{\mathbb{P}_k} \eqsim \inf_{f_e|\Omega = f}|f_e|_{(k+1)} + \|f\|_{L^2(\Omega)}.
    \end{equation}
\end{corollary}
Note that by the remarks in \cite{parhi2020banach}, when $k=1$ this corollary also applies to the semi-norm introduced in \cite{ongie2019function}.
\begin{proof}
 By Theorem \ref{radon-bv-equivalence-thm} we have $$\inf_{f_e|\Omega = f}|f_e|_{(k+1)} = k!\inf_{p\in \mathcal{P}_k}\|f+p\|_{\mathbb{P}_k} \leq k!\|f\|_{\mathbb{P}_k}.$$
 Further, by Lemma \ref{fundamental-norm-lemma} we have $\|f\|_{L^2(\Omega)} \lesssim \|f\|_{\mathbb{P}_k}$ since the dictionary $\mathbb{P}_k$ is bounded in $L^2(\Omega)$. Putting these together, we get
 \begin{equation}
     \inf_{f_e|\Omega = f}|f_e|_{(k+1)} + \|f\|_{L^2(\Omega)} \lesssim \|f\|_{\mathbb{P}_k}.
 \end{equation}
 To prove the other direction, let $p^*\in \mathcal{P}_k$ be such that
 \begin{equation}
     \|f - p^*\|_{\mathbb{P}_k} = \inf_{p\in \mathcal{P}_k}\|f+p\|_{\mathbb{P}_k}.
 \end{equation}
 Such a $p^*$ can always be found since $\mathcal{P}_k$ is a finite dimensional space and the function $\|f + p\|_{\mathbb{P}_k}\rightarrow \infty$ as $p\rightarrow \infty$. Then, using Theorem \ref{radon-bv-equivalence-thm}, we have
 \begin{equation}
     \|f\|_{\mathbb{P}_k} \leq \|f - p^*\|_{\mathbb{P}_k} + \|p^*\|_{\mathbb{P}_k} = \frac{1}{k!}\inf_{f_e|\Omega = f}|f_e|_{(k+1)} + \|p^*\|_{\mathbb{P}_k}.
 \end{equation}
 Since $\mathcal{K}(\mathbb{P}_k)$ contains all polynomials of degree at most $k$ on $\Omega$, $\|\cdot\|_{\mathbb{P}_k}$ is a finite norm on $\mathbb{P}_k$. As all norms on the finite dimensional space $\mathcal{P}_k$ are equivalent we get
 \begin{equation}
     \|p^*\|_{\mathbb{P}_k} \lesssim \|p^*\|_{L^2(\Omega)}.
 \end{equation}
 Next, we notice that
 \begin{equation}
     \|p^*\|_{L^2(\Omega)} \leq  \|f\|_{L^2(\Omega)} + \|f - p^*\|_{L^2(\Omega)}.
 \end{equation}
 Further, $\|f - p^*\|_{L^2(\Omega)} \lesssim \|f\|_{\mathbb{P}_k} \lesssim \inf_{f_e|\Omega = f}|f_e|_{(k+1)}$ for a constant $C$. This follows by Lemma \ref{fundamental-norm-lemma}, Theorem \ref{radon-bv-equivalence-thm} and the fact that $|p|_{(k+1)} = 0$ for any $p\in \mathcal{P}_k$ (Lemma 19 in \cite{parhi2020banach}). It follows that
 \begin{equation}
     \|p^*\|_{\mathbb{P}_k} \lesssim\|p^*\|_{L^2(\Omega)} \leq \|f\|_{L^2(\Omega)} + \|f - p^*\|_{L^2(\Omega)} \lesssim \inf_{f_e|\Omega = f}|f_e|_{(k+1)} + \|f\|_{L^2(\Omega)},
 \end{equation}
 so we finally get
 \begin{equation}
     \|f\|_{\mathbb{P}_k} \leq\frac{1}{k!}\inf_{f_e|\Omega = f}|f_e|_{(k+1)} + \|p^*\|_{\mathbb{P}_k} \lesssim \inf_{f_e|\Omega = f}|f_e|_{(k+1)} + \|f\|_{L^2(\Omega)}.
 \end{equation}
\end{proof}

\section{Characterization of $\mathcal{K}(\mathbb{P}_k)$ in One Dimension}\label{one-dimensional-characterization}
In this section, we prove a characterization of $\mathcal{K}(\mathbb{P}_k)$ in one dimension. In this case, the space $\mathcal{K}(\mathbb{P}_k)$ has a relatively simple characterization in terms of the space of bounded variation. In the case where $k=1$ an analogous characterization can be found in \cite{wojtowytsch2020representation}, section 4. Earlier results characterizing the Barron space in one dimension on the while of $\mathbb{R}$ were obtained in \cite{savarese2019infinite,li2020complexity}. Note that by the results of the previous section, a higher dimensional characterization in terms of the Radon transform is given in \cite{ongie2019function,parhi2020banach}.

\begin{theorem}\label{barron-space-1-d-characterization-theorem}
 Let $\Omega = [-1,1]$. We have
 \begin{equation}
 \mathcal{K}(\mathbb{P}_k) = \{f\in L^2([-1,1]):~\text{$f$ is $k$-times differentiable a.e. and }f^{(k)}\in BV([-1,1])\}.
 \end{equation}
 In particular, it holds that
 \begin{equation}
  \|f\|_{\mathbb{P}_k} \eqsim \sum_{j=0}^{k-1} |f^{(j)}(-1)| + \|f^{(k)}\|_{BV([-1,1])}. 
 \end{equation}

\end{theorem}
\begin{proof}
 We first prove that 
 \begin{equation}\label{upper-bound-barron-1-d}
  \|f\|_{\mathbb{P}_k} \lesssim \sum_{j=0}^{k-1} |f^{(j)}(-1)| + \|f^{(k)}\|_{BV([-1,1])}. 
 \end{equation}
Note that the right hand side is uniformly bounded for all $f = \sigma_k(\pm x + b)\in \mathbb{P}_k$, since $\sigma_k^{(k)}$ is a multiple of the Heaviside function and $b$ is bounded by $\max(|c_1|,|c_2|)$. By taking convex combinations, this means that for some constant $C$, we have
\begin{equation}
 \left\{\sum_{j=1}^n a_jh_j:~h_j\in \mathbb{P}^1_k,~\sum_{i=1}^n|a_i|\leq 1\right\} \subset CB^1_{BV,k},
\end{equation}
where
\begin{equation}
 B^1_{BV,k}:=\left\{f\in L^2([-1,1]):~\sum_{j=0}^{k-1} |f^{(j)}(-1)| + \|f^{(k)}\|_{BV([-1,1])} \leq 1\right\}.
\end{equation}
It is well-known that $B^1_{BV,k}$ is compact in $L^1([-1,1])$ (see, for instance Theorem 4 of Chapter 5 in \cite{evans2015measure}). This implies that $B^1_{BV,k}$ is closed in $L^2([-1,1])$, since if $f_n\rightarrow_{L^2} f$ with $f_n\in B^1_{BV,k}$, then there must exist a subsequence $f_{k_n}\rightarrow_{L^1} \tilde{f}\in B^1_{BV,k}$. Clearly $f=\tilde{f}$ and so $B^1_{BV,k}$ is closed in $L^2([-1,1])$. From this it follows that $\overline{\conv(\pm\mathbb{P}_k)} \subset CB^1_{BV,k}$ and we obtain \eqref{upper-bound-barron-1-d}.

Next, we prove the reverse inequality. So let $f\in B^1_{BV,k}$. By Theorem 2 in Chapter 5 of \cite{evans2015measure}, there exist $f_n\in C^\infty\cap B^1_{BV,k}$ such that $f_n\rightarrow f$ in $L^1([-1,1])$. Further, since $f_n,f\in B^1_{BV,k}$, we have that $\|f - f_n\|_{L^\infty([-1,1])}$ is uniformly bounded. Thus $$\|f - f_n\|^2_{L^2([-1,1])} \leq \|f - f_n\|_{L^1([-1,1])}\|f - f_n\|_{L^\infty([-1,1])} \rightarrow 0$$
and so $f_n\rightarrow f$ in $L^2([-1,1])$ as well.

Using the Peano kernel formula, we see that
\begin{equation}
 f_n(x) = \sum_{j=0}^{k} \frac{f_n^{(j)}(-1)}{j!}(x+1)^j + \int_{-1}^1 \frac{f_n^{(k+1)}(b)}{k!}\sigma_k(x-b)db.
\end{equation}
From the definition of the $BV$-norm and the fact that $f_n\in B^1_{BV,k}$, we see that
\begin{equation}
 \sum_{j=0}^{k} \frac{|f_n^{(j)}(-1)|}{j!}+ \int_{-1}^1 \frac{|f_n^{(k+1)}(b)|}{k!}db \leq C_1
\end{equation}
for a fixed constant $C_1$. Choose $k+1$ distinct $b_1,...,b_{k+1}\in [1, c_2]$ (note that we need $c_2 > 1$). Then by construction $\sigma_k(x+b_i) = (x+b_i)^k$ is a polynomial on $[-1,1]$. Moreover, it is well-known that the polynomials $(x+b_i)^k$ span the space of polynomials of degree at most $k$ (using for instance the determinant of Vandermonde matrix). Combined with the coefficient bound
\begin{equation}
 \sum_{j=0}^{k} \frac{|f_n^{(j)}(-1)|}{j!} \leq C_1,
\end{equation}
we see that
\begin{equation}
 \sum_{j=0}^{k} \frac{f_n^{(j)}(-1)}{j!}(x-a)^j \in C_2\cdot\overline{\conv(\pm\mathbb{P}_k)}
\end{equation}
for a fixed constant $C_2$ (independent of $f_n$). Furthermore, since also
\begin{equation}
 \int_{-1}^1 \frac{|f_n^{(k+1)}(b)|}{k!}db \leq C_1,
\end{equation}
we obtain
\begin{equation}
 \int_{-1}^1 \frac{f_n^{(k+1)}(b)}{k!}\sigma_k(x-b)db\in C_1\cdot\overline{\conv(\pm\mathbb{P}_k)}.
\end{equation}
This implies that $f_n\in C\cdot\overline{\conv(\pm\mathbb{P}_k)}$ for $C = C_1 + C_2$ and since $f_n\rightarrow f$ and $\overline{\conv(\pm\mathbb{P}_k)}$ is closed in $L^2([-1,1])$, we get $f\in C\cdot\overline{\conv(\pm\mathbb{P}_k)}$, which completes the proof.

\end{proof}

\section{Characterization of $\mathcal{K}_1(\mathbb{F}^d_s)$}\label{spectral-barron-section}
In this section we characterize the space $\mathcal{K}(\mathbb{F}_s)$ and the variation norm corresponding to the dictionary $\mathbb{F}_s$. In particular, we have that this variation norm is equivalent to the spectral Barron norm which has been widely used in the approximation theory of shallow neural networks \cite{barron1993universal,siegel2020approximation,siegel2020high,siegel2021sharp,klusowski2018approximation}.
\begin{theorem}\label{spectral-barron-theorem}
We have
\begin{equation}\label{fourier-integral-condition}
 \|f\|_{\mathbb{F}^d_s} = \inf_{f_e|_{\Omega} = f} \int_{\mathbb{R}^d} (1+|\xi|)^s|\hat{f}_e(\xi)|d\xi,
\end{equation}
where the infimum is taken over all extensions $f_e\in L^1(\mathbb{R}^d)$.
\end{theorem}
Note that we have equality in the above theorem, not just equivalence of the norms.
We remark that throughout this section, we use the following convention for the Fourier transform
\begin{equation}
 \hat{f}(\xi) = \int_{\mathbb{R}^d} f(x)e^{-2\pi \iu \xi\cdot x}dx,
\end{equation}
for which the inverse transform is given by
\begin{equation}\label{inverse-transform}
 f(x) = \int_{\mathbb{R}^d} \hat{f}(\xi)e^{2\pi \iu    \xi\cdot x}d\xi.
\end{equation}
To prove Theorem \ref{spectral-barron-theorem} we will need the following technical lemma concerning cutoff functions.

\begin{lemma}\label{fourier-cutoff-lemma}
  Suppose that $\Omega\subset \mathbb{R}^d$ is bounded. Let $\epsilon > 0$ and $s\geq 0$. Then there exists a function $\phi\in L^1(\mathbb{R}^d)$, such that $\phi(x) = 1$ for $x\in \Omega$ and 
  \begin{equation}
  \int_{\mathbb{R}^d}(1+|\xi|)^s|\hat{\phi}(\xi)|d\xi \leq 1 + \epsilon.
  \end{equation}
 \end{lemma}
 \begin{proof}
  Since $\Omega$ is bounded, it suffices to consider the case where $\Omega = [-L,L]^d$ for a sufficiently large $L$. We consider separable $\phi = \phi_1(x_1)\cdots\phi_d(x_d)$, and note that
  \begin{equation}
   \int_{\mathbb{R}^d}(1+|\xi|)^s|\hat{\phi}(\xi)|d\xi \leq \int_{\mathbb{R}^d}\prod_{i=1}^d(1+|\xi_i|)^s|\hat{\phi}_i(\xi_i)|d\xi \leq \prod_{i=1}^d \int_{\mathbb{R}}(1+|\xi|)^s|\hat{\phi}_i(\xi)|d\xi,
  \end{equation}
  and this reduces us to the one-dimensional case where $\Omega = [-L,L]$.
  
  For the one-dimensional case, consider a Gaussian $g_R(x) = e^{-\frac{x^2}{2R}}$. A simple calculation shows that the Fourier transform of the Gaussian is $\hat{g}_R(\xi) = \sqrt{\frac{R}{2\pi}}e^{-\frac{R\xi^2}{2}}$. This implies that
  \begin{equation}
   \lim_{R\rightarrow \infty} \int_{\mathbb{R}}(1+|\xi|)^s|\hat{g}_R(\xi)|d\xi = 1,
  \end{equation}
  and thus by choosing $R$ large enough, we can make this arbitrarily close to $1$.
  
  Now consider $\tau_R\in C^{k}(\mathbb{R})$ for $k > s+2$ such that $\tau_R(x) = 1 - g_R(x)$ for $x\in [-L,L]$. Then we have 
  $$\|\tau_R\|_{L^\infty([-L,L])}, \|\tau_R^\prime\|_{L^\infty([-L,L])}, \cdots, \|\tau_R^{(k)}\|_{L^\infty([-L,L])} \rightarrow 0$$
  as $R\rightarrow \infty$.
 Consequently, it is possible to extend $\tau_R$ to $\mathbb{R}$ so that
 \begin{equation}
  \|\tau_R\|_{L^1(\mathbb{R})}, \|\tau_R^{(k)}\|_{L^1(\mathbb{R})} \rightarrow 0.
 \end{equation}
 as $R\rightarrow \infty$. For instance, for $x > L$ we can take $\tau_R$ to be a polynomial which matches the first $k$ derivatives at $L$ times a fixed smooth cutoff function which is identically $1$ in some neighborhood of $L$ (and similarly at $-L$).
 
 This implies that $\|\hat{\tau}_R(\xi)\|_{L^\infty(\mathbb{R})},\|\xi^{k}\hat{\tau}_R(\xi)\|_{L^\infty(\mathbb{R})}\rightarrow 0$ as $R\rightarrow \infty$. Together, these imply that
 \begin{equation}
  \lim_{R\rightarrow \infty} \int_{\mathbb{R}}(1+|\xi|)^s|\hat{\tau}_R(\xi)|d\xi \rightarrow 0,
 \end{equation}
 since $k-2 > s$.
 
 Finally, set $\phi_R = g_R(x) + \tau_R(x)$. Then clearly $\phi_R = 1$ on $[-L,L]$ and also
 \begin{equation}
  \lim_{R\rightarrow \infty} \int_{\mathbb{R}}(1+|\xi|)^s|\hat{\phi}_R(\xi)|d\xi \leq \lim_{R\rightarrow \infty} \int_{\mathbb{R}}(1+|\xi|)^s|\hat{\tau}_R(\xi)|d\xi + \lim_{R\rightarrow \infty} \int_{\mathbb{R}}(1+|\xi|)^s|\hat{g}_R(\xi)|d\xi = 1.
 \end{equation}
 Choosing $R$ large enough, we obtain the desired result.
\end{proof}

Using this lemma, we now show that integral representations of the form \eqref{integral-representation-def} over the dictionary $\mathbb{F}_s$ are equivalent to the right hand side of \eqref{fourier-integral-condition}.
\begin{proposition}\label{key-fourier-proposition}
 Let $\Omega\subset \mathbb{R}^d$ be a bounded domain and $s \geq 0$. Then
 \begin{equation}\label{barron-norm-form-2}
  \inf\left\{\|\mu\|:~f = \int_{\mathbb{F}_s} i_{\mathbb{F}_s\rightarrow L^2(\Omega)}d\mu\right\} = \inf_{f_e|_{\Omega} = f} \int_{\mathbb{R}^d} (1+|\xi|)^s|\hat{f}_e(\xi)|d\xi.
 \end{equation}

\end{proposition}
\begin{proof}
 We first prove the inequality
 \begin{equation}
     \inf\left\{\|\mu\|:~f = \int_{\mathbb{F}_s} i_{\mathbb{F}_s\rightarrow L^2(\Omega)}d\mu\right\} \leq \inf_{f_e|_{\Omega} = f} \int_{\mathbb{R}^d} (1+|\xi|)^s|\hat{f}_e(\xi)|d\xi.
 \end{equation}
 If the right hand side is infinite, there is nothing to prove. So let $f_e\in L^1(\mathbb{R}^d)$ be an extension such that
 \begin{equation}
     \int_{\mathbb{R}^d} (1+|\xi|)^s|\hat{f}_e(\xi)|d\xi < \infty.
 \end{equation}
 In particular, this means that $\hat{f}\in L^1(\mathbb{R}^d)$ as well and Fourier inversion holds almost everywhere. So we get 
 \begin{equation}
     f(x) = \int_{\mathbb{R}^d}(1+|\xi|)^{-s}e^{2\pi\iu\xi\cdot x}(1+|\xi|)^s\hat{f}(\xi)d\xi
 \end{equation}
 for almost every $x\in \Omega$.
 Thus, by choosing $\mu = (1+|\xi|)^s\hat{f}(\xi)d\xi$  we get
 \begin{equation}
     f = \int_{\mathbb{F}_s} i_{\mathbb{F}_s\rightarrow L^2(\Omega)}d\mu,
 \end{equation}
 where the Bochner integral in is justified since $\mathbb{F}_s$ is uniformly bounded in $L^2(\Omega)$ and $\|\mu\| < \infty$. The right hand side is then a function in $L^2(\Omega)$ which agrees with $f$ almost everywhere (hence we have equality). Thus we get
 \begin{equation}
  \inf\left\{\|\mu\|:~f = \int_{\mathbb{F}_s} i_{\mathbb{F}_s\rightarrow L^2(\Omega)}d\mu\right\} \leq \inf_{f_e|_{\Omega} = f} \int_{\mathbb{R}^d} (1+|\xi|)^s|\hat{f}_e(\xi)|d\xi.
 \end{equation}

 Now let us prove the reverse inequality. Let $\lambda$ be a regular Borel measure such that the integral on the right hand side of \eqref{barron-norm-form-2} is finite (note this must mean that $\lambda$ has finite mass) and 
 \begin{equation}
  f(x)=\int_{\mathbb{F}_s} i_{\mathbb{F}_s\rightarrow L^2(\Omega)}d\lambda=\int_{\mathbb{R}^d} (1+|\xi|)^{-s}e^{2\pi \iu \xi\cdot x}d\lambda(\xi)
 \end{equation}
 for $x\in \Omega$. Let $\mu = (1+|\xi|)^{-s}\lambda$, so that we have
 \begin{equation}
     f(x)=\int_{\mathbb{R}^d} e^{2\pi \iu \xi\cdot x}d\mu(\xi)
 \end{equation}
 and
 \begin{equation}
     \int_{\mathbb{R}^d}(1+|\nu|)^s d|\mu|(\nu) = \|\lambda\|.
 \end{equation}
 Choose $\epsilon > 0$. By Lemma \ref{fourier-cutoff-lemma} we can find a $\phi\in L^1(\mathbb{R}^d)$ such $\phi|_\Omega = 1$ and $$\int_{\mathbb{R}^d}(1+|\xi|)^s|\hat{\phi}(\xi)|d\xi \leq 1 + \epsilon.$$ 
 We now set
 \begin{equation}
  f_e(x) = \phi(x)\left[\int_{\mathbb{R}^d} e^{2\pi \iu \xi\cdot x}d\mu(\xi)\right]\in L^1(\mathbb{R}^d),
 \end{equation}
 since $\phi\in L^1(\mathbb{R}^d)$ and $\mu$ has finite mass, so the second factor must be bounded.

 Then we have that for $x\in \Omega$,
 \begin{equation}
  f(x) = f(x)\phi(x) = f_e(x),
 \end{equation}
 and $\hat{f}_e = \hat{\phi} * \mu$,
 where the function $\hat{\phi} * \mu$ is given by
 \begin{equation}
  (\hat{\phi} * \mu)(\xi) = \int_{\mathbb{R}^d} \hat{\phi}(\xi - \nu)d\mu(\nu).
 \end{equation}
 We now calculate
 \begin{equation}
  \int_{\mathbb{R}^d}(1+|\xi|)^s|(\hat{\phi} * \mu)(\xi)|d\xi \leq \int_{\mathbb{R}^d}\int_{\mathbb{R}^d}(1+|\xi|)^s |\hat{\phi}(\xi - \nu)|d|\mu|(\nu)d\xi.
 \end{equation}
 Finally, we use the simple inequality $(1+|\xi|)^s \leq (1+|\nu|)^s(1+|\xi - \nu|)^s$ combined with a change of variables, to get
 \begin{equation}
 \begin{split}
 \int_{\mathbb{R}^d}(1+|\xi|)^s|(\hat{\phi} * \mu)(\xi)|d\xi &\leq \left(\int_{\mathbb{R}^d}(1+|\xi|)^s|\hat{\phi}(\xi)|d\xi\right)\left(\int_{\mathbb{R}^d}(1+|\nu|)^s d|\mu|(\nu)\right)\\
 &\leq (1+\epsilon)\left(\int_{\mathbb{R}^d}(1+|\nu|)^s d|\mu|(\nu)\right) = (1+\epsilon)\|\lambda\|.
 \end{split}
 \end{equation}
 This shows that
 \begin{equation}
  \inf_{f_e|_{\Omega} = f} \int_{\mathbb{R}^d} (1+|\xi|)^s|\hat{f}_e(\xi)|d\xi \leq (1+\epsilon)\inf\left\{\|\mu\|:~f = \int_{\mathbb{F}_s} i_{\mathbb{F}_s\rightarrow L^2(\Omega)}d\mu\right\}.
 \end{equation}
 Since $\epsilon > 0$ was arbitrary, we get the desired result.
\end{proof}

This completes the proof of Theorem \ref{spectral-barron-theorem} if $s > 0$ since then $\mathbb{F}_s$ is compact in $L^2(\Omega)$ and we can invoke Lemma \ref{prokhorov-lemma} to obtain the equality
\begin{equation}
    \|f\|_{\mathbb{F}_s} = \int_{\mathbb{F}_s} i_{\mathbb{F}_s\rightarrow L^2(\Omega)}d\mu
\end{equation}
 The final step is thus to prove the left equality in \ref{fourier-integral-condition} when $s=0$. For this, we use the following.

\begin{proposition}
 Let $\Omega\subset \mathbb{R}^d$ be a bounded domain. Then
 \begin{equation}\label{spectral-barron-integral-condition}
  B_e(\Omega) = \left\{f:\Omega\rightarrow \mathbb{R}:~\inf_{f_e|_\Omega = f} \int_{\mathbb{R}^d} |\hat{f}_e(\xi)|d\xi\leq 1\right\}
\end{equation}
is closed in $L^2(\Omega)$.
\end{proposition}
\begin{proof}
 Let $f_n\rightarrow f$ in $L^2(\Omega)$ with $f_n\in B_e(\Omega)$. Choose $\epsilon > 0$ and consider the corresponding sequence of $h_n = \hat{f}_{n,e}$ in \eqref{spectral-barron-integral-condition} which satisfy
 \begin{equation}\label{eq-870}
  \int_{\mathbb{R}^d}|h_n(\xi)|d\xi \leq 1 + \epsilon,~f_n(x)=\hat{h}_n(x) = \int_{\mathbb{R}^d} h_n(\xi)e^{2\pi \iu \xi\cdot x}d\xi.
 \end{equation}
 By assumption $f_n\rightarrow f$ in $L^2(\Omega)$ so that for any $g\in L^2(\Omega)$, we have
 \begin{equation}
  \langle f_n, g\rangle_{L^2(\Omega)} \rightarrow \langle f, g\rangle_{L^2(\Omega)}.
 \end{equation}
 Choose $g$ to be any element in the dense subset $C^\infty_c(\Omega)\subset L^2(\Omega)$ and note that in this case we have by Plancherel's theorem
 \begin{equation}\label{eq-878}
  \langle f_n, g\rangle_{L^2(\Omega)} = \langle f_n, g\rangle_{L^2(\mathbb{R}^d)} = \langle h_n, \hat{g}\rangle_{L^2(\mathbb{R}^d)}.
 \end{equation}
 Note that $\hat{g}$ is a Schwartz function and so is in $C_0(\mathbb{R}^d)$, the space of continuous, decaying functions
 \begin{equation}
  C_{0}(\mathbb{R}^d) = \{\phi\in C(\mathbb{R}):\lim_{\xi\rightarrow \infty} |\phi(\xi)| = 0\},
 \end{equation}
 with the supremum norm.
 
 This implies that the map
 \begin{equation}
  h:\phi \rightarrow \lim_{n\rightarrow \infty}\langle h_n, \phi\rangle_{L^2(\mathbb{R}^d)} 
 \end{equation}
 defines a bounded linear functional on the subspace of $C_{0}(\mathbb{R}^d)$ which is spanned by $\{\hat{g}:g\in C^\infty_c(\Omega)\}$. The limit above exists by \eqref{eq-878} and the assumption that $f_n\rightarrow f$. Further, the bound has norm $\leq 1 + \epsilon$ by equation \eqref{eq-870}.
 
 By the Hahn-Banach theorem, we can extend $h$ to an element $\mu\in C^*_{0}(\mathbb{R}^d)$, such that $\|\mu\|_{C^*_{0}(\mathbb{R}^d)}\leq 1 + \epsilon$. By the Riesz-Markov theorem (Theorem 22 in \cite{markoff1938mean}), the dual space $C^*_{0}(\mathbb{R}^d)$ is exactly the space of Borel measures with the total variation norm. Thus we get
 \begin{equation}
  \|\mu\|_{C^*_{0}(\mathbb{R}^d)} = \int_{\mathbb{R}^d} d|\mu|(\xi) \leq 1 + \epsilon.
 \end{equation}
 But we also have that for every $g\in C^\infty_c(\Omega)$, $\langle \mu, \hat{g}\rangle = \langle f,g\rangle$. Taking the Fourier transform, we see that the function
 \begin{equation}
  f_\mu = \int_{\mathbb{R}^d}e^{2\pi \iu  \xi\cdot x}d\mu(\xi)
 \end{equation}
 satisfies $\langle f_\mu, g\rangle = \langle f,g\rangle$ for all $g\in C^\infty_c(\Omega)$. Thus $f = f_\mu$ in $L^2(\Omega)$ and so by \eqref{barron-norm-form-2}, we have 
 $$\inf_{f_e|_\Omega = f} \int_{\mathbb{R}^d} |\hat{f}_e(\xi)|d\xi \leq \int_{\mathbb{R}^d} |\hat{f}_{\mu}(\xi)|d\xi \leq 1 + \epsilon.$$ 
 Since $\epsilon$ was arbitrary, this completes the proof.

\end{proof}
To complete the proof in the case of $s=0$, we simply note that by \eqref{barron-norm-form-2}, $B_e(\Omega)$ contains all of the complex exponentials $e^{2\pi i\omega\cdot x}$. Since it is clearly convex and is closed by Proposition \ref{key-fourier-proposition}, it must be equal to $\overline{\conv(\pm\mathbb{F}_0)}$. This completes the proof of Theorem \ref{spectral-barron-theorem}.

\section{Conclusion}
We have provided some foundational analysis of the variation spaces with respect to dictionaries arising in the study of shallow neural networks. The precise analysis of approximation theoretic properties such as the metric entropy and $n$-widths of these spaces is a major research direction which we propose. In addition, it must be investigated whether these variation spaces are useful for any particular practical applications.
\section{Acknowledgements}
We would like to thank Professors Russel Caflisch, Ronald DeVore, Weinan E, Albert Cohen, Stephan Wojtowytsch and Jason Klusowski for helpful discussions. We would also like to thank the anonymous reviewers for their helpful comments. This work was supported by the Verne M. Willaman Chair Fund at the Pennsylvania State University, and the National Science Foundation (Grant No. DMS-1819157).

\bibliographystyle{spmpsci}
\bibliography{refs}

\begin{thebibliography}{10}
\providecommand{\url}[1]{{#1}}
\providecommand{\urlprefix}{URL }
\expandafter\ifx\csname urlstyle\endcsname\relax
  \providecommand{\doi}[1]{DOI~\discretionary{}{}{}#1}\else
  \providecommand{\doi}{DOI~\discretionary{}{}{}\begingroup
  \urlstyle{rm}\Url}\fi

\bibitem{barron1993universal}
Barron, A.R.: Universal approximation bounds for superpositions of a sigmoidal
  function.
\newblock IEEE Transactions on Information theory \textbf{39}(3), 930--945
  (1993)

\bibitem{barron2008approximation}
Barron, A.R., Cohen, A., Dahmen, W., DeVore, R.A.: Approximation and learning
  by greedy algorithms.
\newblock The annals of statistics \textbf{36}(1), 64--94 (2008)

\bibitem{devore1998nonlinear}
DeVore, R.A.: Nonlinear approximation.
\newblock Acta numerica \textbf{7}, 51--150 (1998)

\bibitem{devore1996some}
DeVore, R.A., Temlyakov, V.N.: Some remarks on greedy algorithms.
\newblock Advances in computational Mathematics \textbf{5}(1), 173--187 (1996)

\bibitem{diestel2012sequences}
Diestel, J.: Sequences and series in Banach spaces, vol.~92.
\newblock Springer Science \& Business Media (2012)

\bibitem{dudley2018real}
Dudley, R.M.: Real analysis and probability.
\newblock CRC Press (2018)

\bibitem{ma2019barron}
E, W., Ma, C., Wu, L.: Barron spaces and the compositional function spaces for
  neural network models.
\newblock arXiv preprint arXiv:1906.08039  (2019)

\bibitem{wojtowytsch2020representation}
E, W., Wojtowytsch, S.: Representation formulas and pointwise properties for
  barron functions.
\newblock CoRR  (2020)

\bibitem{evans2015measure}
Evans, L.C., Gariepy, R.F.: Measure theory and fine properties of functions.
\newblock CRC press (2015)

\bibitem{hao2021efficient}
Hao, W., Jin, X., Siegel, J.W., Xu, J.: An efficient greedy training algorithm
  for neural networks and applications in pdes.
\newblock arXiv preprint arXiv:2107.04466  (2021)

\bibitem{hornik1994degree}
Hornik, K., Stinchcombe, M., White, H., Auer, P.: Degree of approximation
  results for feedforward networks approximating unknown mappings and their
  derivatives.
\newblock Neural Computation \textbf{6}(6), 1262--1275 (1994)

\bibitem{jones1992simple}
Jones, L.K.: A simple lemma on greedy approximation in hilbert space and
  convergence rates for projection pursuit regression and neural network
  training.
\newblock The annals of Statistics \textbf{20}(1), 608--613 (1992)

\bibitem{kainen2010integral}
Kainen, P.C., Kurkov{\'a}, V., Vogt, A.: Integral combinations of heavisides.
\newblock Mathematische Nachrichten \textbf{283}(6), 854--878 (2010)

\bibitem{klusowski2018approximation}
Klusowski, J.M., Barron, A.R.: Approximation by combinations of relu and
  squared relu ridge functions with $\ell^1$ and $\ell^0$ controls.
\newblock IEEE Transactions on Information Theory \textbf{64}(12), 7649--7656
  (2018)

\bibitem{krogh1991simple}
Krogh, A., Hertz, J.: A simple weight decay can improve generalization.
\newblock Advances in neural information processing systems \textbf{4} (1991)

\bibitem{kurkova2001bounds}
Kurkov{\'a}, V., Sanguineti, M.: Bounds on rates of variable-basis and
  neural-network approximation.
\newblock IEEE Transactions on Information Theory \textbf{47}(6), 2659--2665
  (2001)

\bibitem{kurkova2002comparison}
Kurkov{\'a}, V., Sanguineti, M.: Comparison of worst case errors in linear and
  neural network approximation.
\newblock IEEE Transactions on Information Theory \textbf{48}(1), 264--275
  (2002)

\bibitem{lecun2015deep}
LeCun, Y., Bengio, Y., Hinton, G.: Deep learning.
\newblock Nature \textbf{521}(7553), 436--444 (2015)

\bibitem{leshno1993multilayer}
Leshno, M., Lin, V.Y., Pinkus, A., Schocken, S.: Multilayer feedforward
  networks with a nonpolynomial activation function can approximate any
  function.
\newblock Neural networks \textbf{6}(6), 861--867 (1993)

\bibitem{li2020complexity}
Li, Z., Ma, C., Wu, L.: Complexity measures for neural networks with general
  activation functions using path-based norms.
\newblock arXiv preprint arXiv:2009.06132  (2020)

\bibitem{livshits2009lower}
Livshits, E.D.: Lower bounds for the rate of convergence of greedy algorithms.
\newblock Izvestiya: Mathematics \textbf{73}(6), 1197 (2009)

\bibitem{makovoz1996random}
Makovoz, Y.: Random approximants and neural networks.
\newblock Journal of Approximation Theory \textbf{85}(1), 98--109 (1996)

\bibitem{makovoz1998uniform}
Makovoz, Y.: Uniform approximation by neural networks.
\newblock Journal of Approximation Theory \textbf{95}(2), 215--228 (1998)

\bibitem{markoff1938mean}
Markoff, A.: On mean values and exterior densities.
\newblock Rec. Math. \textbf{4(46)}(1), 165--191 (1938)

\bibitem{nair2010rectified}
Nair, V., Hinton, G.E.: Rectified linear units improve restricted boltzmann
  machines.
\newblock In: ICML (2010)

\bibitem{ongie2019function}
Ongie, G., Willett, R., Soudry, D., Srebro, N.: A function space view of
  bounded norm infinite width relu nets: The multivariate case.
\newblock In: International Conference on Learning Representations (ICLR 2020)
  (2019)

\bibitem{parhi2020banach}
Parhi, R., Nowak, R.D.: Banach space representer theorems for neural networks
  and ridge splines.
\newblock arXiv preprint arXiv:2006.05626  (2020)

\bibitem{parhi2021kinds}
Parhi, R., Nowak, R.D.: What kinds of functions do deep neural networks learn?
  insights from variational spline theory.
\newblock arXiv preprint arXiv:2105.03361  (2021)

\bibitem{petrosyan2020neural}
Petrosyan, A., Dereventsov, A., Webster, C.G.: Neural network integral
  representations with the relu activation function.
\newblock In: Mathematical and Scientific Machine Learning, pp. 128--143. PMLR
  (2020)

\bibitem{pisier1981remarques}
Pisier, G.: Remarques sur un r{\'e}sultat non publi{\'e} de b. maurey.
\newblock S{\'e}minaire Analyse fonctionnelle (dit ``Maurey-Schwartz") pp.
  1--12 (1981)

\bibitem{prokhorov1956convergence}
Prokhorov, Y.V.: Convergence of random processes and limit theorems in
  probability theory.
\newblock Theory of Probability \& Its Applications \textbf{1}(2), 157--214
  (1956)

\bibitem{savarese2019infinite}
Savarese, P., Evron, I., Soudry, D., Srebro, N.: How do infinite width bounded
  norm networks look in function space?
\newblock In: Conference on Learning Theory, pp. 2667--2690. PMLR (2019)

\bibitem{siegel2020approximation}
Siegel, J.W., Xu, J.: Approximation rates for neural networks with general
  activation functions.
\newblock Neural Networks \textbf{128}, 313--321 (2020)

\bibitem{siegel2020high}
Siegel, J.W., Xu, J.: \protect{High-Order Approximation Rates for Neural
  Networks with {ReLU$^k$} Activation Functions}.
\newblock arXiv preprint arXiv:2012.07205  (2020)

\bibitem{siegel2021sharp}
Siegel, J.W., Xu, J.: Sharp bounds on the approximation rates, metric entropy,
  and $n$-widths of shallow neural networks.
\newblock arXiv preprint arXiv:2101.12365  (2021)

\bibitem{sil2004rate}
Sil'nichenko, A.: Rate of convergence of greedy algorithms.
\newblock Mathematical Notes \textbf{76}(3), 582--586 (2004)

\bibitem{temlyakov2011greedy}
Temlyakov, V.: Greedy approximation, vol.~20.
\newblock Cambridge University Press (2011)

\bibitem{temlyakov2008greedy}
Temlyakov, V.N.: Greedy approximation.
\newblock Acta Numerica \textbf{17}(235), 409 (2008)

\bibitem{weinan2019barron}
Weinan, E., Ma, C., Wu, L.: Barron spaces and the compositional function spaces
  for neural network models.
\newblock arXiv preprint arXiv:1906.08039  (2019)

\bibitem{weinan2021barron}
Weinan, E., Ma, C., Wu, L.: The barron space and the flow-induced function
  spaces for neural network models.
\newblock Constructive Approximation pp. 1--38 (2021)

\bibitem{CiCP-28-1707}
Xu, J.: Finite neuron method and convergence analysis.
\newblock Communications in Computational Physics \textbf{28}(5), 1707--1745
  (2020).
\newblock \doi{https://doi.org/10.4208/cicp.OA-2020-0191}.
\newblock
  \urlprefix\url{http://global-sci.org/intro/article_detail/cicp/18394.html}

\end{thebibliography}

\end{document}